\documentclass[letterpaper]{article}
\usepackage{aaai17}
\usepackage{times}
\usepackage{helvet}
\usepackage{courier}
\usepackage[utf8]{inputenc} 
\usepackage[T1]{fontenc}    
\usepackage{url}            
\usepackage{booktabs}       
\usepackage{amsfonts}       
\usepackage{nicefrac}       

\usepackage{amsfonts}
\usepackage{amsmath}
\usepackage{wrapfig}

\usepackage{scabby}

\usepackage{algorithm}
\usepackage[noend]{algcompatible}

\usepackage{graphicx}
\usepackage{subfigure}
\usepackage{caption}

\title{Estimating Uncertainty Online Against an Adversary}
\author{
  Volodymyr Kuleshov \\
  Stanford University\\
  Stanford, CA 94305 \\
  \texttt{kuleshov@cs.stanford.edu} \\
  \And
  Stefano Ermon \\
  Stanford University \\
  Stanford, CA 94305 \\
  \texttt{ermon@cs.stanford.edu} 
}

\providecommand{\e}{\epsilon}

\newcommand{\ellavg}{{\bar \ell}}

\newcommand{\Ic}{{\mathcal I}}

\newcommand{\Fc}{{\mathcal F}}

\newcommand{\Rb}{{\mathbb R}}
\newcommand{\extR}{R^\mathrm{ext}}
\newcommand{\intR}{R^\mathrm{int}}
\newcommand{\Fcal}{F^\mathrm{cal}}
\newcommand{\Ind}{{\mathbb{I}}}
\newcommand{\Exp}{{\mathbb{E}}}

\newcommand{\palg}{p^F}
\newcommand{\wsupj}{\Ind^{(j)}}

\newcommand{\rjt}{{\rho^{(j)}_T}}

\newtheorem{defn}{Definition}

\newcommand\regmin{{\sc RegMin}}

\begin{document}

\maketitle

\newcommand{\fix}{\marginpar{FIX}}
\newcommand{\new}{\marginpar{NEW}}

\begin{abstract}
\begin{quote}

Assessing uncertainty is an important step towards ensuring the safety and reliability of machine learning systems.
Existing uncertainty estimation techniques may fail when their modeling assumptions are not met, e.g.~when the data distribution differs from the one seen at training time.
%
%
%
Here, we propose techniques that 
assess a classification algorithm's uncertainty via calibrated probabilities 
(i.e.~probabilities that match empirical outcome frequencies in the long run) and which are guaranteed to be reliable (i.e.~accurate and calibrated) on out-of-distribution input, including input generated by an adversary. 
This represents an extension of classical online learning that handles uncertainty in addition to guaranteeing accuracy under adversarial assumptions.
%
We establish formal guarantees for our methods, and we validate them on two real-world problems: question answering and medical diagnosis from genomic data. 
\end{quote}
\end{abstract}

\section{Introduction}\label{sec:introduction}

Assessing uncertainty is an important step towards ensuring the safety and reliability of machine learning systems.
In many applications of machine learning --- 
including medical diagnosis \cite{jiang2012calibrating}, natural language understanding \cite{nguyen2015posterior}, 
and speech recognition \cite{dong2011calibration} 
--- assessing confidence can be as important as obtaining high accuracy.
This work explores confidence estimation for classification problems.


An important limitation of existing methods is the assumption that data is sampled i.i.d.~from a distribution $\BP(x,y)$; when test-time data is distributed according to a different $\BP^*$, these methods may become overconfident and erroneous.
Here, we introduce new, robust uncertainty estimation algorithms guaranteed to produce reliable confidence estimates on out-of-distribution input, including input generated by an adversary. 

In the classification setting,
the most natural way of measuring an algorithm's uncertainty is 
via {\em calibrated} probability estimates that match the true empirical frequencies of an outcome. For example, if an algorithm predicted a 60\% chance of rain 100 times in a given year, its forecast would be calibrated if it rained on about 60 of those 100 days. 

\paragraph{Background.}

Calibrated confidence estimates are typically constructed via {\em recalibration}, using methods such as Platt scaling \cite{platt1999probabilistic} or isotonic regression \cite{niculescu2005predicting}. In the context of binary classification, these methods reduce recalibration to a one-dimensional regression problem that, given data $(x_i, y_i)_{i=1}^n$, trains a model $g(s)$ (e.g.~logistic regression) to predict probabilities $p_i = g(s_i)$ from uncalibrated scores $s_i = h(x_i)$ produced by a classifier $h$ (e.g. ~SVM margins). Fitting $g$ is equivalent to performing density estimation targeting $\BP(Y=1 | h(X) = s_i)$ and hence may fail on out-of-distribution testing data.

The methods we introduce in this work are instead based on calibration techniques developed in the literature on online learning in mathematical games \cite{foster98asymptoticcalibration,abernethy11blackwell}. 
These {\em classical} methods are not suitable for standard prediction tasks in their current form. For one, they do not admit covariates $x_i$ that might be available to improve the prediction of $y_i$; hence, they also do not consider the predictive power of the forecasts.
For example, predicting 0.5 on a sequence 01010... formed by alternating 0s and 1s is considered a valid calibrated forecaster. The algorithms we present here combine the advantages of online calibration (adversarial assumptions), and of batch probability recalibration (covariates and forecast sharpness). 

\paragraph{Online learning with uncertainty.}

Whereas classical online optimization aims to accurately predict targets $y$ given $x$ (via a convex loss $\ell(x,y)$), our algorithms aim to accurately predict {\em uncertainties} $p(y=\hat y)$.
The $p$ here are defined as empirical frequencies over data seen so far; it turns out that these probability-like quantities  can be estimated under the standard adversarial assumptions of online learning.
We thus see our work as extending classical online optimization to handle uncertainty in addition to guaranteeing accuracy.

\paragraph{Example.}

As a concrete motivating example, consider a medical system that diagnoses a long stream of patients indexed by $t=1,2,...$, outputting a disease risk $p_t \in [0,1]$ for each patient based on their medical record $x_t$. 
Provably calibrated probabilities in this setting may be helpful for making informed policy decisions (e.g.~by providing guaranteed upper bounds on the number of patients readmitted after a discharge) and may be used to communicate risks to patients in a more intuitive way. This setting is also inherently online, since patients are typically observed one at a time, and may not be i.i.d.~due to e.g., seasonal disease outbreaks.
\paragraph{Contributions.} More formally, our contributions are to:

\begin{itemize}
\item Formulate a new problem called {\em online recalibration}, 
which requires producing calibrated probabilities on potentially adversarial input, while retaining the predictive power of a given baseline uncalibrated forecaster.

\item Propose a meta-algorithm for online recalibration that uses classical online calibration as a black box subroutine.

\item Show that our technique can recalibrate the forecasts of any existing classifier at the cost of an $O(1/\sqrt{\e})$ overhead in the convergence rate of $\mathcal{A}$, where $\e > 0$ is the desired level of accuracy. 

\item Surprisingly, both online and standard batch recalibration (e.g., Platt scaling) may be performed only when accuracy is measured using specific loss functions; our work characterizes the losses which admit a recalibration procedure in both the online and batch settings.
\end{itemize}

\section{Background}\label{sec:background}

Below, we will use $\Ind_E$ denote the indicator function of $E$, $[N]$ and $[N]_0$ to (respectively) denote the sets $\{1,2,...,N\}$ and $\{0,1,2,...,N\}$, and $\Delta_d$ to denote the $d$-dimensional simplex.

\subsection{Learning with Expert Advice}\label{sec:advice}

Learning with expert advice \cite{cesabianchi2006prediction} is a special case of the general online optimization framework \cite{shalev2007phd} that underlies online calibration algorithms.
At each time $t=1,2,...$, the forecaster $F$ receives advice from $N$ {\em experts} and chooses a distribution $w_t \in \Delta_{N-1}$ over their advice. Nature then reveals an outcome $y_t$ and $F$ incurs an expected loss of $\sum_{i=1}^N w_{ti} \ell(y_t, a_{it})$, where $\ell(y_t, a_{it})$ is the loss under expert $i$'s advice $a_{it}$. Performance in this setting is measured using two notions of regret.
%
%
\begin{defn}
The external regret $\extR_T$ and the internal regret $\intR_T$ are defined as
\begin{align*}
\extR_T & = \sum_{t=1}^T \ellavg(y_t, p_t)  - \min_{i \in [N]} \sum_{t=1}^T \ell(y_t, a_{it}) \\
\intR_T & = \max_{i,j \in [N]} \sum_{t=1}^T p_{t,i} \left( \ell(y_t, a_{it})  - \ell(y_t, a_{jt}) \right),
\end{align*}
where $ \ellavg(y, p) = \sum_{i=1}^N p_i \ell(y, a_{it}) $ is the expected loss.
\end{defn}

External regret measures loss with respect to the best fixed expert, while internal regret is a stronger notion
that measures the gain from retrospectively switching all the plays of action $i$ to $j$.
Both definitions admit algorithms with sublinear, uniformly bounded regret.

In this paper, we will be particularly interested in {\em proper} losses $\ell$, whose expectation over $y$ is minimized by the probability corresponding to the average $y$.
\begin{defn}
A loss $\ell(y, p) : \{0,1\} \times [0,1]  \to \mathbb{R}_+$ is proper if
$p \in \arg\min_q \Exp_{y \sim \text{Ber}(p)} \ell(y, q) \; \forall p.$ 
\end{defn}
Examples of proper losses include the L2 loss $\ell_2(y,p) = (y-p)^2$,
the log-loss $\ell_\text{log}(y,p) = y\log(p) + (1-y)\log(1-p)$,
and the 
the misclassification loss $\ell_\text{mc}(y,p) = (1-y) \Ind_{p < 0.5} + y \Ind_{p \geq 0.5}$. Counter-examples include the L1 and the hinge losses.

\begin{figure}
\begin{center}
\includegraphics[width=8.5cm]{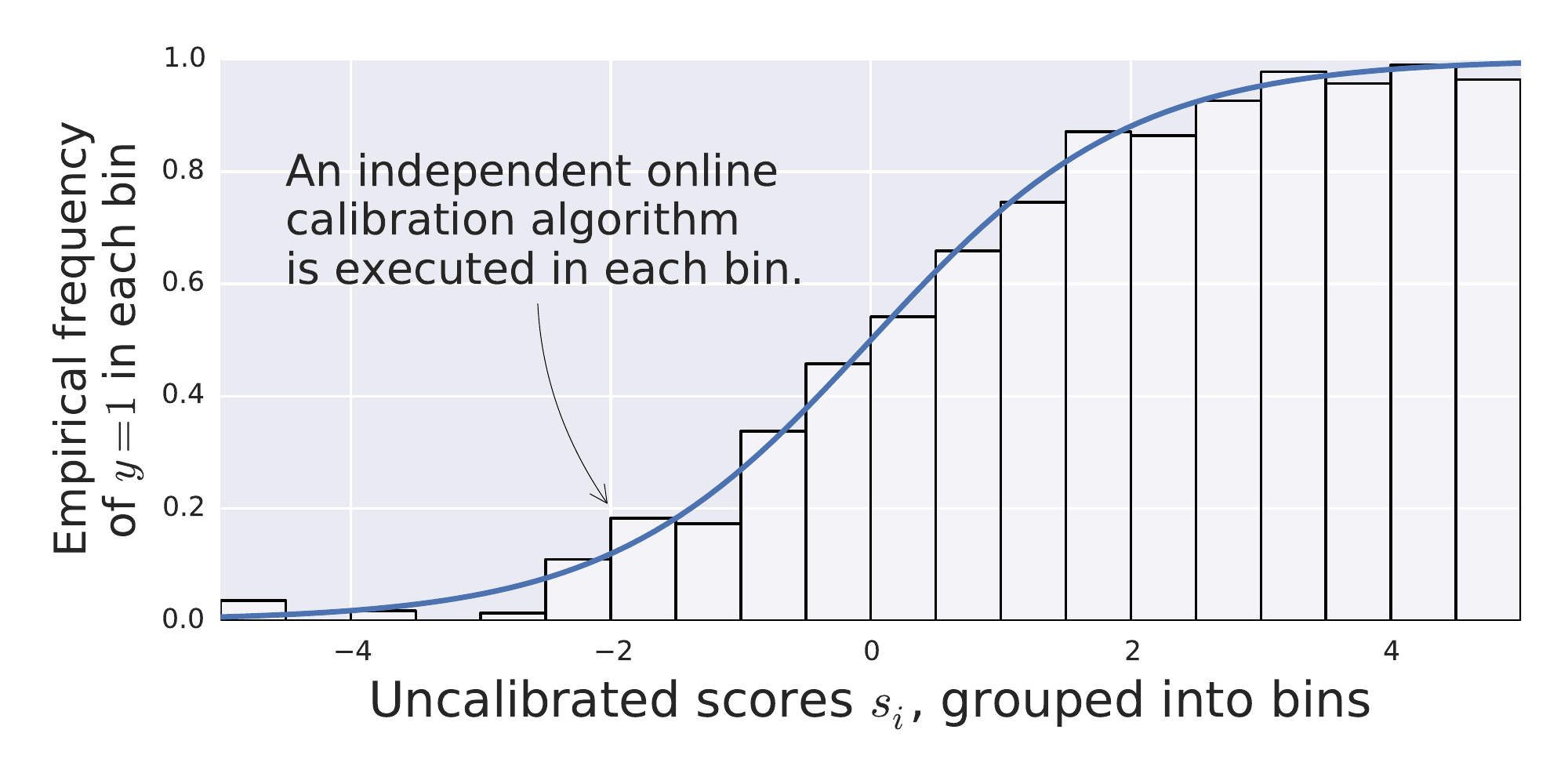}
\end{center}
\caption{Our method bins uncalibrated scores and runs online calibration subroutines in each bin (not unlike  the histogram recalibration method targeting $\mathbb{P}(y=1\mid s=t)$).}
\end{figure}

\subsection{Calibration in Online Learning}

Intuitively, calibration means that the true and predicted frequencies of an event should match. For example, if an algorithm predicts a 60\% chance of rain 100 times in a given year, then we should see rain on about 60 of those 100 days.
%
More formally, let $\Fcal$ be a forecaster making predictions in the set $\{\frac{i}{N} \mid i=0,...,N\}$,  where $1/N$ is called the {\em resolution} of $\Fcal$;
consider the quantities $\rho_T(p) = \dfrac{\sum_{t=1}^T y_t \Ind_{p_t = p}}{\sum_{t=1}^T \Ind_{p_t = p}}$ and
\begin{align}
C_T^p & = \sum_{i=0}^N \left| \rho_T(i/N) - \frac{i}{N} \right|^p \left( \frac{1}{T} \sum_{t=1}^T \Ind_{\{p_t = \frac{i}{N}\}} \right). \label{eqn:cal_loss}
\end{align}
The term $\rho_T(p)$ denotes the frequency at which event $y = 1$ occurred over the times when we predicted $p$. Our intuition was that $\rho_T(p)$ and $p$ should be close to each other; we capture this using the notion of calibration error $C_T^p$
for $p \geq 1$; this corresponds to the weighted $\ell_p$ distance between the $\rho_T(i/N)$ and the predicted probabilities $\frac{i}{N}$; typically one assumes that $p=1$ or $p=2$. To simplify notation, we will use the term $C_T$ when the exact $p$ is unambiguous.

\begin{defn}
We say that $\Fcal$ is an $(\e, \ell_p)$-calibrated algorithm with resolution $1/N$ if
$ \lim \sup_{T \to \infty} C_{T}^p \leq \e $ a.s.
\end{defn}




There exists a vast literature on calibration in the online setting \cite{cesabianchi2006prediction} which is primarily concerned with constructing calibrated predictions $p_t \in [0,1]$ of a binary outcome $y_t \in \{0,1\}$ based solely on the past sequence $y_1,...,y_{t-1}$.
Surprisingly, this is possible even when the $y_t$ are chosen adversarially by reducing the problem to
internal regret minimization relative to $N+1$ experts with losses $(y_t - i/N)^2$ and proposed predictions $i/N$ for $i \in [N]_0$. 
All such algorithms are randomized, hence our results will hold almost surely (a.s.).
See Chapter 4 in \citeauthor{cesabianchi2006prediction} for details.


\section{Online Recalibration}\label{sec:recalibration}


Unfortunately, existing online calibration methods are not directly applicable in real-world settings.
For one, they do not take into account covariates $x_t$ that might be available to improve the prediction of $y_t$. As a consequence, they cannot produce accurate forecasts: for example, they would constantly predict 0.5 on a sequence 01010... formed by alternating 0s and 1s.

To address these shortcomings, we define here a new problem called {\em online recalibration}, in which the task is to 
transform a sequence of uncalibrated forecasts $\palg_t$ into predictions $p_t$ that are calibrated and almost as accurate as the original $\palg_t$. The forecasts $\palg_t$ may come from any existing machine learning system $F$; our methods treat it as a black box and preserve its favorable convergence properties.

Formally, we define the online recalibration task as a generalization of the classical online optimization framework \cite{shalev2007phd,cesabianchi2006prediction}. At every step $t=1,2,...$:
  \begin{algorithmic}[1]
    \STATE Nature reveals features $x_t \in \mathbb R^d$.
    \STATE Forecaster $F$ predicts $\palg_t = \sigma (w_{t-1} \cdot x_t) \in [0,1]$.
    \STATE A recalibration algorithm $A$ produces a calibrated probability $p_t = A(\palg_t) \in [0,1]$.
    \STATE Nature reveals label $y_t \in \{0, 1\}$; $F$ incurs loss of $\ell(y_t , p_t)$, where $\ell : [0,1] \times \{0, 1\} \to \Rb^+$  is convex in $p_t$ for all $y_t$.
    \STATE $F$ chooses $w_{t+1}$; $A$ updates itself based on $y_t$.
  \end{algorithmic}
Here, $\sigma$ is a transfer function chosen such that the task is convex in $w_t$.
In the medical diagnosis example, $x_t$ represents medical or genomic features for patient $t$; we use feature weights $w_t$ to predict the probability $\palg_t$ that the patient is ill; the true outcome is encoded by $y_t$. We would like $A$ to produce $\palg_t$ that are accurate and well-calibrated in the following sense.
\begin{defn}
We say that $A$ is an $(\e, \ell^\textrm{cal})$-accurate online recalibration algorithm for the loss $\ell^\textrm{acc}$ if (a) the forecasts $p_t = A(\palg_t)$ are $(\e, \ell^\textrm{cal})$-calibrated and (b) the regret of $p_t$ with respect to $\palg_t$ is a.s.~small in terms of $\ell^\textrm{acc}$:  
\begin{align}
\lim\sup_{T \to \infty} \frac{1}{T} \sum_{t=1}^T \left( \ell^\textrm{acc}(y_t , p_t) - \ell^\textrm{acc}(y_t, \palg_t)\right) \leq \epsilon.
\end{align}
\end{defn}

\section{Algorithms for Online Recalibration}\label{sec:framework}

\begin{figure}
\vspace{-3mm}
\begin{algorithm}[H]
  \caption{Online Recalibration}
  \label{algo:recal}
  \begin{algorithmic}[1]
    \REQUIRE Online calibration subroutine $\Fcal$ and number of buckets $M$
    \STATE Let $\Ic = \{[0,\frac{1}{M}), [\frac{1}{M}, \frac{2}{M}), ..., [\frac{M-1}{M},1]\}$ be a set of intervals that partition $[0,1]$.
    \STATE Let $\Fc = \{ \Fcal_j \mid j = 0,...,M-1 \}$ be a set of $M$ independent instances of $\Fcal.$
    \FOR {$t=1,2,...$:}
    \STATE Observe uncalibrated forecast $\palg_t$.
    \STATE Let $I_j \in \Ic$ be the interval containing $\palg_t$.
    \STATE Let $p_t$ be the forecast of $\Fcal_j$. 
    \STATE Output $p_t$. Observe $y_t$ and pass it to $\Fcal_j$.
    \ENDFOR
  \end{algorithmic}
\end{algorithm}
\vspace{-7mm}
\end{figure}

Next, we propose an algorithm for performing online probability recalibration; we refer to our approach as a meta-algorithm because it repeatedly invokes a regular online calibration algorithm as a black-box subroutine. \algorithmref{recal} outlines this procedure.

At a high level, \algorithmref{recal} partitions the uncalibrated forecasts $\palg_t$ into $M$ buckets/intervals $\Ic = \{[0,\frac{1}{M}), [\frac{1}{M}, \frac{2}{M}), ..., [\frac{M-1}{M},1]\}$; it trains an independent instance of $\Fcal$ on the data $\{\palg_t, y_t \mid \palg_t \in I_j \}$ belonging to each bucket $I_j \in \Ic$; at prediction time, it calls the instance of $\Fcal$ associated with the bucket of the uncalibrated forecast $\palg_t$.

\algorithmref{recal} works because a calibrated predictor is at least as accurate as any constant predictor; in particular, each subroutine $\Fcal_j$ is at least as accurate as the prediction $\frac{j}{M}$, which also happens to be approximately $\palg_t$ when $\Fcal_j$ was called. Thus, each $\Fcal_j$ is as accurate as its input sequence of $\palg_t$. One can then show that if each each $\Fcal_j$ is accurate and calibrated, then so it their aggrgate, \algorithmref{recal}. The rest of this section provides a formal version of this argument; due to space limitations, we defer most of our full proofs to the appendix.


\subsection{Calibration and Accuracy of Online Recalibration}


\paragraph{Notation.}

We define the calibration error of $\Fcal_j$ and of \algorithmref{recal} at $i/N$ as (respectively)
{\footnotesize
\begin{align*}
C^{(j)}_{T,i}  = \left| \rjt(i/N) - \frac{i}{N} \right|^p \left( \frac{1}{T_j} \sum_{t=1}^T \wsupj_{t,i} \right) \\ 
C_{T,i} = \left|  \rho_T(i/N) - \frac{i}{N} \right|^p \left( \frac{1}{T} \sum_{t=1}^T \Ind_{t,i} \right),
\end{align*}}

\noindent where $\Ind_{t,i} = \Ind\{p_t = i/N\}$. Terms marked with a $(j)$ denote the restriction of the usual definition to the input of subroutine $\Fcal_j$ (see the appendix for details).
We may write the calibration losses of $\Fcal_j$ and \algorithmref{recal} 
as $ C^{(j)}_{T} = \sum_{i=0}^N C^{(j)}_{T,i}$ and $ C_{T} = \sum_{i=0}^N C_{T,i}$.

\paragraph{Assumptions.}

In this section, we will assume that the subroutine $\Fcal$ used in \algorithmref{recal} is $(\e, \ell_1)$-calibrated and that $C^{(j)}_{T_j} \leq R_{T_j} + \e$ uniformly ($R_{T_j} = o(1)$ as $T_j \to \infty$; $T_j$ is the number of calls to instance $\Fcal_j$). This also implies $\ell_p$-calibration (by continuity of $\ell_p$), albeit with different rates $R_{T_j}$ and a different $\e$. \citeauthor{abernethy11blackwell} introduce $(\e, \ell_1)$-calibrated $F_j$. We also provide proofs for the $\ell_2$ loss in the appendix.

Crucially, we assume that the loss $\ell$ used for measuring accuracy is {\em proper} and bounded with $\ell(\cdot,i/N) < B$ for $i \in [N]_0$; since the set of predictions is finite, this is a mild requirement. Finally, we make additional continuity assumptions on $\ell$ in \lemmaref{accuracy}.

\paragraph{Recalibration with proper losses.}

Surprisingly, not every loss $\ell$ admits a recalibration procedure. Consider, for example, the following continuously repeating sequence $001001001...$ of $y_t$'s. A calibrated forecaster must converge to predicting $1/3$ (a constant prediction) with an $\ell_1$ loss of $\approx$0.44; however predicting $0$ for all $t$ has an $\ell_1$ loss of $1/3 < 0.44$. Thus we cannot recalibrate this sequence and also remain equally accurate under the $\ell_1$ loss. The same argument also applies to batch recalibration (e.g. Platt scaling): we only need to assume that $y_t \sim \mathrm{Ber}(1/3)$ i.i.d.

However, recalibration is possible for a very large class of {\em proper} losses. Establishing this fact will rely on the following key technical lemma.

\begin{lemma}\label{lem:noregret}
If $\ell$ is a proper loss bounded by $B>0$, then an $(\e, \ell_1)$-calibrated $\Fcal$
a.s.~has a small internal regret w.r.t.~$\ell$ and satisfies uniformly over time $T$ the bound
{\footnotesize
\begin{align*}
\intR_{T} & = \max_{ij} \sum_{t=1}^T \Ind_{p_t = i/N} \left( \ell(y_t, i/N)  - \ell(y_t, j/N) \right) \leq 2 B (R_T + \e).
\end{align*}
}
\end{lemma}

According to \lemmaref{noregret}, if a set of predictions is calibrated, then we never want to retrospectively switch to predicting $p_2$ at times when we predicted $p_1$. Intuitively, this makes sense: if predictions are calibrated, then $p_1$ should minimize the total (or average) loss $\sum_{t : p_t=p_1 } \ell(y_t, p)$ over the times $t$ when $p_1$ was predicted (at least better so than $p_2$). However, our $\ell_1$ counter-example above shows that this intuition does not hold for every loss; we need to explicitly enforce our intuition, which amounts to assuming that $\ell$ is proper, i.e. that
$p \in \arg\min_q \Exp_{y \sim \text{Ber}(p)} \ell(y, q)$.

\paragraph{Accuracy and calibration.}

An important consequence of \lemmaref{noregret} is that a calibrated algorithm has vanishing regret relative to any fixed prediction (since minimizing internal regret also minimizes external regret). Using this fact, it becomes straightforward to establish that \algorithmref{recal} is at least as accurate as the baseline forecaster $F$.

\begin{lemma}[Recalibration preserves accuracy]\label{lem:accuracy}
Consider \algorithmref{recal} with
parameters $M \geq N > 1/\e$ and let $\ell$ be a bounded proper loss for which
\begin{enumerate}
\item $\ell(y_t, p) \leq \ell(y_t, j/M) + B/M$ for $p \in [j/M, (j+1)/M)$;
\item $\ell(y_t, p) \leq \ell(y_t, i/N) + B/N$ for $p \in [i/N, (i+1)/N)$;
\end{enumerate}
Then the recalibrated $p_t$ a.s.~have vanishing $\ell$-loss regret relative to $\palg_t$ and we have uniformly:
\begin{equation*}
\frac{1}{T} \sum_{t=1}^T \ell (y_t , p_t) - \frac{1}{T} \sum_{t=1}^T \ell(y_t , \palg_t)  < N B \sum_{j=1}^M \frac{T_j}{T} R_{T_j} + 3B\e. 
\end{equation*}
\end{lemma}

\begin{proof}[Proof (sketch)]
When $p_t$ is the output of a given $F_j$, we have $\ell(y_t,p_t^F) \approx \ell(y_t,j/M)\approx \ell(y_t,i_j/M)$ 
(since $\palg_t$ is in the $j$-th bucket, and since $M \geq N$ is sufficiently high resolution).
Since $F_j$ is calibrated, \lemmaref{noregret} implies the $p_t$ have vanishing regret relative to the fixed prediction $i_j/N$; aggregating over $j$ yields our result.
\end{proof}

The assumptions of \lemmaref{accuracy} essentially require that $\ell$ be Lipschitz with constant $B$, which holds e.g.~for convex bounded losses that are studied in online learning. Our assumption is slightly more general since $\ell$ may also be discontinuous (like the misclassification loss).
When $\ell$ is unbounded (like the log-loss), its values at the baseline algorithm's predictions must be bounded away from infinity.

Next, we also establish that combining the predictions of each $\Fcal_j$ preserves their calibration.

\begin{lemma}[Preserving calibration]\label{lem:calibration}
If each $\Fcal_j$ is $(\e, \ell_p)$-calibrated,
then \algorithmref{recal} is also $(\e, \ell_p)$-calibrated and the bound $C_T \leq \sum_{j=1}^M \frac{T_j}{T} R_{T_j} + \e$ holds uniformly over $T$.
\end{lemma}

These two lemmas lead to our main claim: that \algorithmref{recal} solves the online recalibration problem.

\begin{theorem}\label{thm:main}
Let $\Fcal$ be an $(\ell_1, \epsilon/3B)$-calibrated online subroutine with resolution $N \geq 3B/\epsilon$. 
and let $\ell$ be a proper loss satisfying the assumptions of \lemmaref{accuracy}. Then \algorithmref{recal} with parameters $\Fcal$ and $M=N$ is an $\epsilon$-accurate online recalibration algorithm for the loss $\ell$.
\end{theorem}

\begin{proof}
By \lemmaref{calibration}, \algorithmref{recal} is $(\ell_1, \e/3B)$-calibrated and by \lemmaref{accuracy}, its regret w.r.t. the raw $\palg_t$ tends to $< 3B/N < \e$. Hence, \theoremref{main} follows.
\end{proof}

In the appendix, we provide a detailed argument for how $\ell$ can be chosen to be the misclassificaiton loss.

Interestingly, it also turns out that if $\ell$ is not a proper loss, then recalibration is not possible for some $\e > 0$.

\begin{theorem}
If $\ell$ is not proper, then no algorithm achieves recalibration w.r.t.~$\ell$ for all $\e > 0$.
\end{theorem}

The proof of this algorithm is a slight generalization of the counter-example provided for the $\ell_1$ loss. Interestingly, it holds equally for online and batch settings. To our knowledge, it is one of the first characterizations of the limitations of recalibration algorithms.

\paragraph{Convergence rates.}

  \begin{figure}
  \vspace{-6mm}
\begin{table}[H]
\begin{center}
  \def\arraystretch{1.5}
  \begin{tabular}{c | c | c}
Subroutine & {\footnotesize Regret Minimization} & {\footnotesize Blackwell Approchability} \\
\hline
Time {\scriptsize / step} & $O({1}/{\e})$ & $O(\log({1}/{\e}))$ \\
Space {\scriptsize / step} & $O({1}/{\e^2})$ & $O({1}/{\e^2})$ \\
Calibration & $O({1}/{\e \sqrt{\e T}})$ & $O({1}/{\e \sqrt{T}})$ \\
Advantage & Simplicity & Efficiency
\end{tabular}
\end{center}
\vspace{-2mm}
\caption{\footnotesize\label{tbl:rates}Time and space complexity and convergence rate of \algorithmref{recal} using different subroutines.}
\end{table}
    \vspace{-8mm}
  \end{figure}

Next, we are interested in the rate of convergence $R_T$ of the calibration error $C_T$ of \algorithmref{recal}.
For most online calibration subroutines $\Fcal$,
$R_{T} \leq f(\e)/\sqrt{T}$ for some $f(\e)$.
In such cases, we can further bound the calibration error in \lemmaref{calibration} as
$
\sum_{j=1}^M \frac{T_j}{T} R_{T_j} \leq \sum_{j=1}^M \frac{\sqrt{T_j}f(\e)}{T} \leq \frac{f(\e)}{\sqrt{ \e T}}. 
$
In the second inequality, we set the $T_j$ to be equal. 

Thus, our recalibration procedure introduces an overhead of
$ \frac{1}{\sqrt{\e}} $
in the convergence rate of the calibration error $C_T$ and of the regret in \lemmaref{accuracy}.
In addition, \algorithmref{recal} requires $ \frac{1}{{\e}} $ times more memory (we run $1/\e$ instances of $\Fcal_j$), but has the same per-iteration runtime (we activate one $\Fcal_j$ per step).
\tableref{rates} summarizes the convergence rates of \algorithmref{recal} when the subroutine is either the method of \citeauthor{abernethy11blackwell} based on Blackwell approachability or the simpler but slower approach based on internal regret minimization \cite{mannor2010calibration}.

\paragraph{Multiclass prediction.}
In the multiclass setting, we seek a recalibrator $A : \Delta_{K-1} \to \Delta_{K-1}$ producing calibrated probabilities $p_t \in \Delta_{K-1}$ that target class labels $y_t \in \{1,2,...,K\}$.
In analogy to binary recalibration, we may discretize the input space $\Delta_{K-1}$ into a $K$-dimensional grid
and train a classical multi-class calibration algorithm $\Fcal$ \cite{cesabianchi2006prediction} on each subset of $\palg_t$ associated with a cell. 
Just like in the binary setting, a classical calibration method $\Fcal_j$ predicts calibrated $p_t \in \Delta_{K-1}$ based solely on past multiclass labels $y_1,y_2,...,y_{t-1}$; it can serve as a subroutine within \algorithmref{recal}.

However, in the multi-class setting, this construction will require $O(1/\e^K)$ running time per iteration, $O(1/\e^{2K})$ memory, and will have a convergence rate of $O(1/(\e^{2K} \sqrt{T}))$. The exponential dependence on $K$ cannot be avoided, since the calibration problem is fundamentally PPAD-hard \cite{hazan2012calibration}. 
However, there may exist practical workarounds inspired by popular heuristics for the batch setting, such as one-vs-all classification \cite{zadrozny2002transforming}.

  
\begin{figure}
\vspace{-2mm}
\hspace{-5mm}
\includegraphics[width=9.2cm]{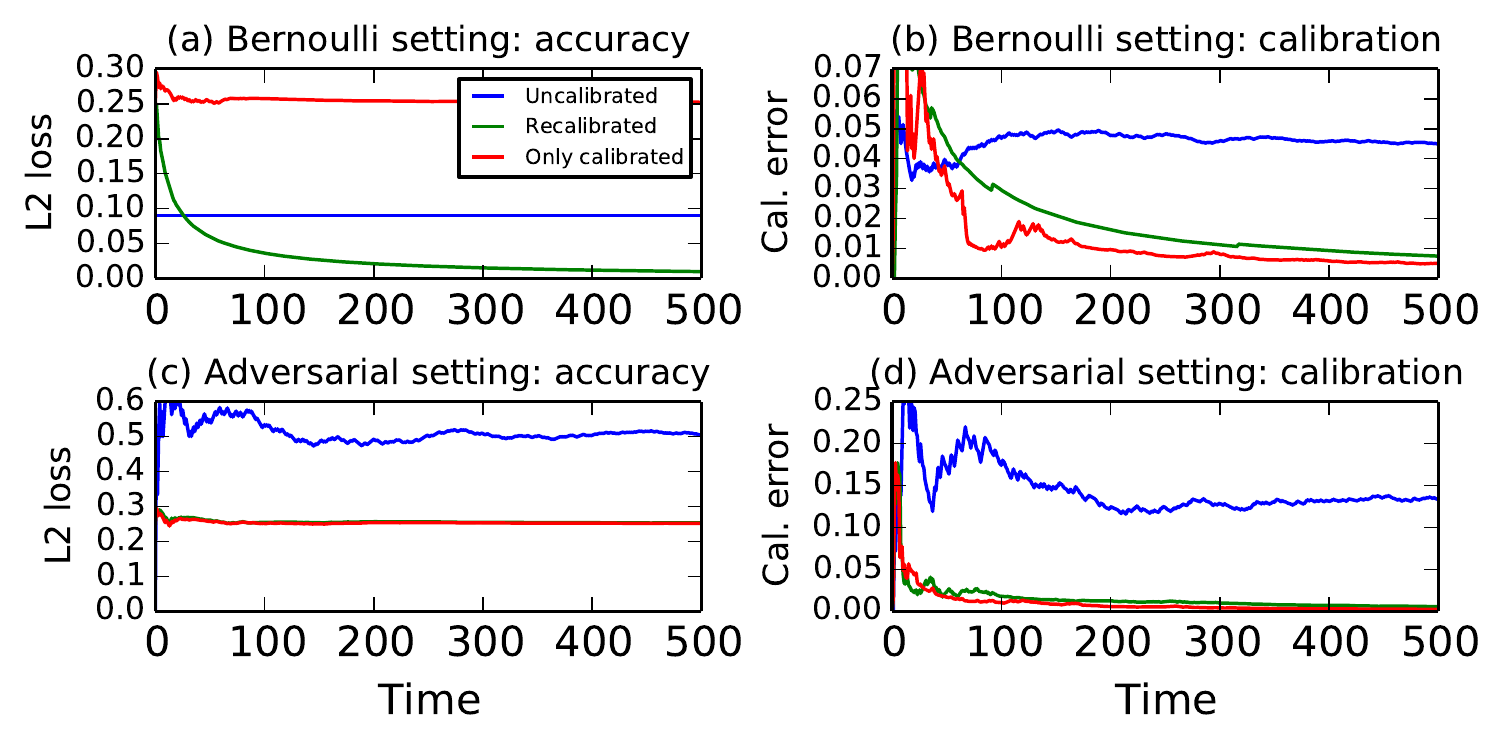}\vspace{-1mm}
\caption{\footnotesize \label{fig:synth}We compare predictions from an uncalibrated expert $F$ (blue), \algorithmref{recal} (green), and \regmin~(red) on sequences $y_t \sim \textrm{Ber}(0.5)$ (plots a, b) and on adversarially chosen $y_t$ (plots c, d). 
}
\vspace{-3mm}
\end{figure}
  
\section{Experiments}\label{sec:experiments}

We now proceed to study \algorithmref{recal} empirically. \algorithmref{recal}'s subroutine is the standard internal regret minimization approach of \citeauthor{cesabianchi2006prediction} ("\regmin"). We measure calibration and accuracy in the $\ell_2$ norm.

\paragraph{Predicting a Bernoulli sequence.}

We start with a simple setting where we observe an i.i.d. sequence of $y_t \sim \textrm{Ber}(p)$ as well as uncalibrated predictions $(\palg_t)_{t=1}^T$ that equal $0.3$ whenever $y_t=0$ and $0.7$ when $y_t=1$. The forecaster $F$ is essentially a perfect predictor, but is not calibrated.

In \figureref{synth}, we compare the performance of \regmin~(which does not observe $\palg_t$) to \algorithmref{recal} and to the uncalibrated predictor $F$. Both methods achieve low calibration error after about 300 observations, while the expert is clearly uncalibrated (\figureref{synth}b); however, \regmin~ is a terrible predictor: it always forecasts $p_t = 0.5$ and therefore has high $\ell_2$ loss (\figureref{synth}a). \algorithmref{recal}, on the other hand, makes perfect predictions by recalibrating the input $\palg_t$.

\paragraph{Prediction against an adversary.}

Next, we test the ability of our method to achieve calibration on adversarial input. At each step $t$, we choose $y_t = 0$ if $p_t > 0.5$ and $y_t = 1$ otherwise; we sample $\palg_t \sim \textrm{Ber}(0.5)$, which is essentially a form of noise. In \figureref{synth} (c, d), we see that \algorithmref{recal} successfully ignores the noisy forecaster $F$ and instead quickly converges to making calibrated (albeit not very accurate) predictions (it reduces to \regmin).

\paragraph{Natural language understanding.}


We used \algorithmref{recal} to recalibrate a state-of-the-art question answering system \cite{berant2014semantic} on the popular Free917 dataset (641 training, 276 testing examples). We trained the system on the training set as described in \cite{berant2013semantic} and then calibrated probabilities using \algorithmref{recal} in one pass over first the training, and then the testing examples. This setup emulates a pre-trained system that further improves itself from user feedback.

\begin{figure*}[t]
\vspace{-2mm}
    \centering
    \begin{subfigure}
{        \centering
	\includegraphics[width=8.5cm]{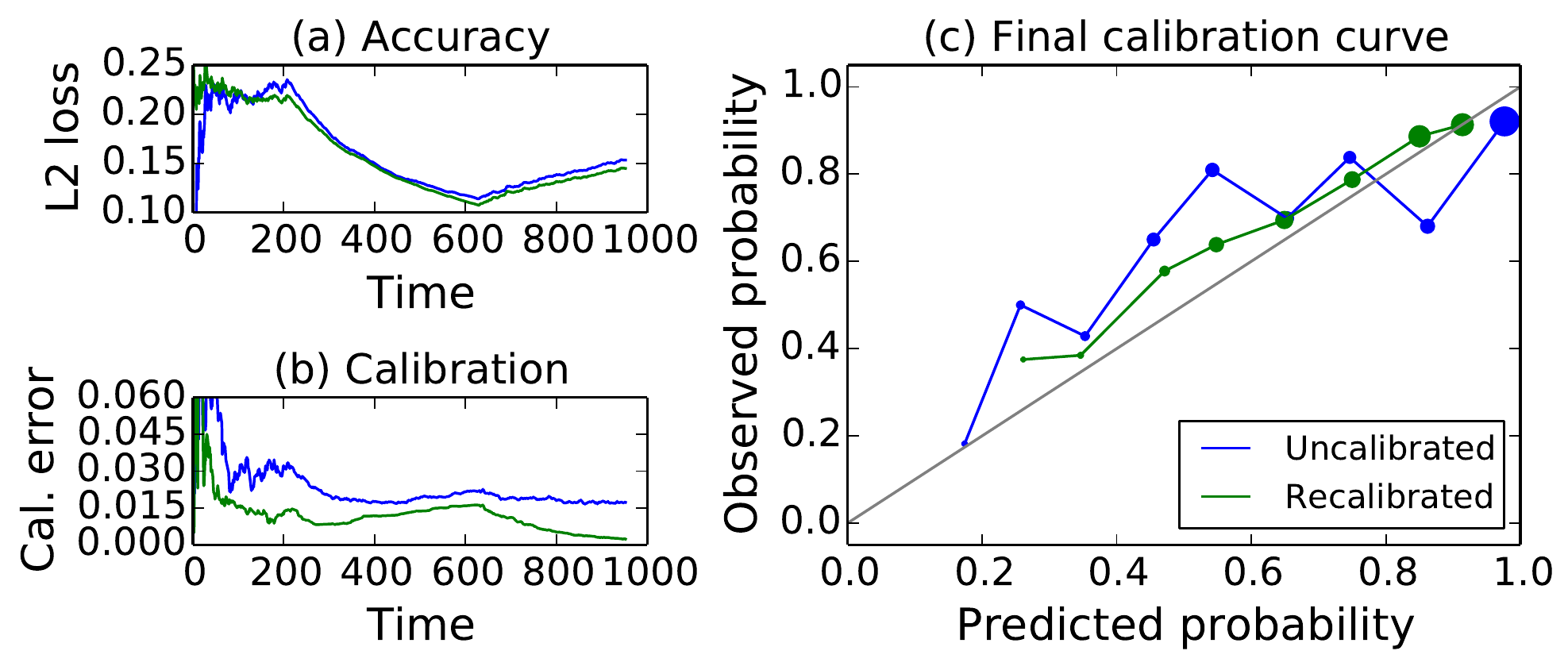}}
    \end{subfigure}\hspace{3mm}
    \begin{subfigure}
{        \centering
\includegraphics[width=8.5cm]{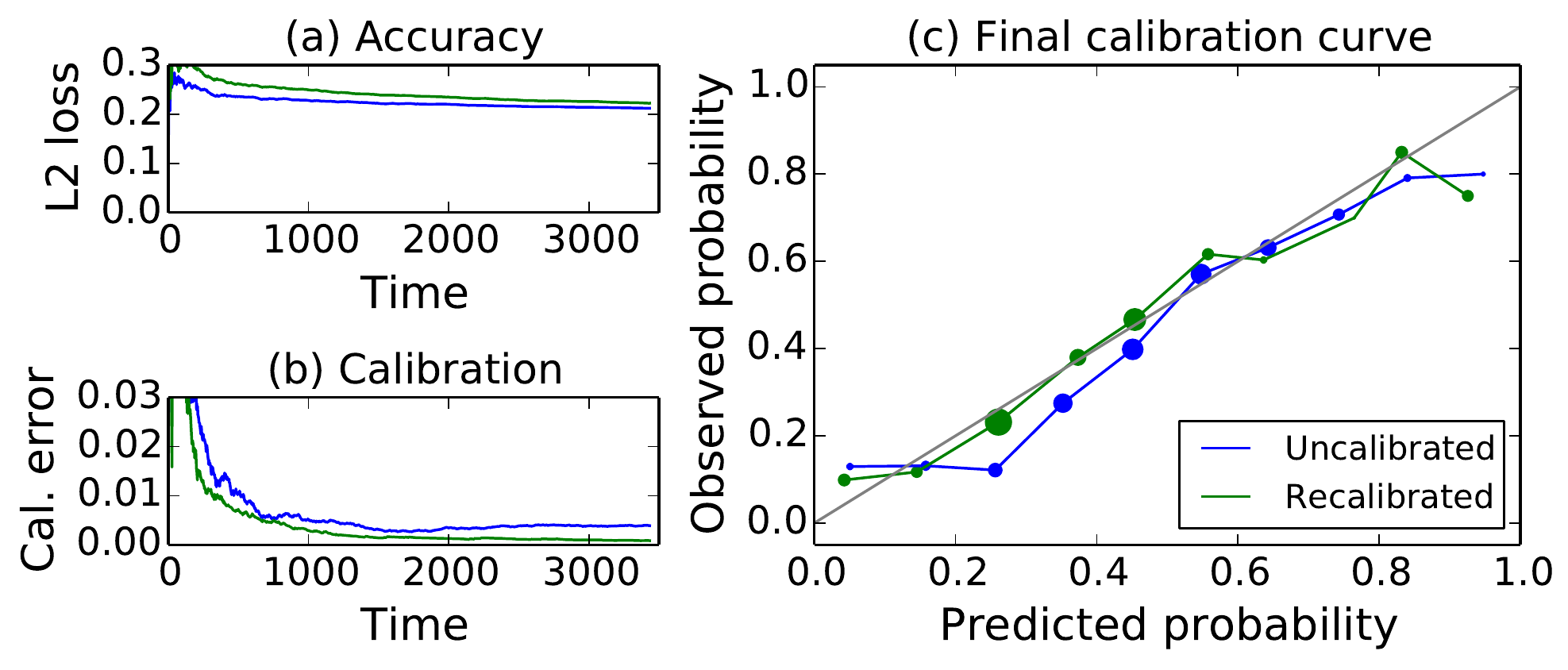}}
    \end{subfigure}\vspace{-1mm}
        \caption{\footnotesize \label{fig:semparse}\algorithmref{recal} (green) is used to recalibrate probabilities from a question answering system (left) and a medical diagnosis system (right; both in blue). We track prediction (a) and calibration error (b) over time; plot (c) displays calibration curves after seeing all the data; circle sizes are proportional to the number of predictions in the corresponding bucket.}
\vspace{-2mm}
\end{figure*}


\figureref{semparse} (left) compares our predicted $p_t$ to the raw system probabilities $\palg_t$ via {\em calibration curves}. Given pairs of predictions and outcomes $p_t, y_t$, we compute for each of $N$ buckets $B \in \{[\frac{i}{N},\frac{i+1}{N}) \mid 0 \leq i \leq 1\}$, averages $\bar p_B = \sum_{t:p_t \in B} p_t /N_B$ and $\bar y_B = \sum_{t:p_t \in B} y_t /N_B$, where $N_B = |\{p_t \in B\}|$. A calibration curve plots the $\bar y_B$ as a function of $\bar p_B$; perfect calibration corresponds to a straight line.

Calibration curves indicate that the $\palg_t$ are poorly calibrated in buckets below 0.9, while \algorithmref{recal} fares better. \figureref{semparse}a confirms that our accuracy (measured by the $\ell_2$ loss) tracks the baseline forecaster.

\paragraph{Medical diagnosis.}

Our last task is predicting the risk of type 1 diabetes from genomic data.
We use genotypes of 3,443 subjects (1,963 cases, 1,480 controls) over 447,221 SNPs \cite{wellcome2007genome}, with alleles encoded as $0,1,2$ (major, heterozygous and minor homozygous resp.). We use an online $\ell_1$-regularized linear support vector machine (SVM) to predict outcomes one patient at a time, and report performance for each $t \in [T]$. Uncalibrated probabilities are normalized raw SVM scores $s_t$, i.e. $\palg_t = (s_t + m_t)/2m_t$, where $m_t = \max_{1 \leq r \leq t} |s_r|$.

\figureref{semparse} (right) measures calibration after observing all the data. Raw scores are not well-calibrated outside of the interval $[0.4, 0.6]$; recalibration makes them almost perfectly calibrated. \figureref{semparse} further shows that the calibration error of \algorithmref{recal} is consistently lower throughout the entire learning process, while accuracy approaches to within $0.01$ of that of $\palg_t$. 


\section{Previous Work}
Calibrated probabilities are widely used as confidence measures in the context of binary classification.
Such probabilities are obtained via recalibration methods, of which Platt scaling \cite{platt1999probabilistic} and isotonic regression \cite{niculescu2005predicting} are by far the most popular. Recalibration methods also possess multiclass extensions, which typically involve training multiple one-vs-all predictors \cite{zadrozny2002transforming}, as well as extensions to ranking losses \cite{menon2012ranking}, combinations of estimators \cite{zhong2013accurate}, and structured prediction \cite{kuleshov2015calibrated}. 

In the online setting, the calibration problem was formalized by \citeauthor{dawid1982well}; online calibration techniques were first proposed by \citeauthor{foster98asymptoticcalibration}. Existing algorithms are based on internal regret minimization \cite{cesabianchi2006prediction} or on Blackwell approachability \cite{foster1997proof}; recently, these approaches were shown to be closely related \cite{abernethy11blackwell,mannor2010calibration}. Recent work has shown that online calibration is PPAD-hard \cite{hazan2012calibration}.

The concepts of calibration and sharpness were first formalized in the statistics literature \cite{murphy1973vector,gneiting2007probabilistic}. These metrics are captured by a class of {\em proper} losses and can be used both for evaluating \cite{buja05lossfunctions,brocker2009decomposition} and constructing \cite{kuleshov2015calibrated} calibrated forecasts.

\section{Discussion and Conclusion}\label{sec:discussion}

\paragraph{Online vs batch.}

\algorithmref{recal} can be understood as a direct analogue of a simple density estimation technique called the histogram method. This technique divides the $\palg_t$ into $N$ bins and estimates the average $y$ in each bin. By the i.i.d.~assumption, output probabilities will be calibrated; sharpness will be determined by the bin width.
%
Note that by Hoeffding's inequality, the average in a given bin with converge at a rate of $O({1}/{\sqrt{T_j}})$ \cite{devroye1996probabilistic}. 
This is faster than the $O({1}/{\sqrt{\e T_j}})$ rate of \citeauthor{abernethy11blackwell} and suggests that calibration is more challenging in the online setting. 

\paragraph{Checking rules.} 
An alternative way to avoid uninformative predictions (e.g.~ 0.5 on 010101...) is via the framework of {\em checking rules} \cite{cesabianchi2006prediction}. However, these rules must be specified in advance (e.g.~ the pattern 010101 must be known) and this framework does not explicitly admit covariates $x_t$. Our approach on the other hand recalibrates any $x_t, y_t$ in a black-box manner.

\paragraph{Defensive forecasting.} 

\citeauthor{vovk2005defensive} developed simultaneously calibrated and accurate online learning methods under the notion of {\em weak} calibration \cite{abernethy2011efficient}. We use strong calibration, which implies weak, although it requires different (e.g. randomized) algorithms. Vovk et al.~also use a different notion of precision; their algorithm ensures a small difference between average predicted $p_t$ and true $y_t$ at times $t$ when $p_t \approx p^*$ and $x_t \approx x^*$, for any $p^*, x^*$. The relation $\approx$ is determined by a user-specified kernel (over e.g. sentences or genomes $x_t$). Our approach, on the other hand, does not require specifying a kernel, and matches the accuracy of any given baseline forecaster; this may be simpler in some settings. 
Interestingly, we arrive at the same rates of convergence under different assumptions.

\paragraph{Conclusion.}


Current recalibration techniques implicitly require that the data is distributed i.i.d., which potentially makes them unreliable when this assumption does not hold. 
In this work, we introduced the first recalibration technique that provably recalibrates any existing forecaster with a vanishingly small degradation in accuracy. This method does not make i.i.d.~assumptions, and is provably calibrated even on adversarial input. We analyzed our method's theoretical properties and showed excellent empirical performance on several real-world benchmarks, where the method converges quickly and retains good accuracy.

\paragraph{Acknowledgements.} This work is supported by the NSF (grant \#1649208) and by the Future of Life Institute (grant 2016-158687).

\newpage

{
\bibliographystyle{aaai}
\bibliography{all}

\begin{thebibliography}{}

\bibitem[\protect\citeauthoryear{Abernethy and
  Mannor}{2011}]{abernethy2011efficient}
Abernethy, J.~D., and Mannor, S.
\newblock 2011.
\newblock Does an efficient calibrated forecasting strategy exist?
\newblock In {\em {COLT} 2011 - The 24th Annual Conference on Learning Theory,
  June 9-11, 2011, Budapest, Hungary},  809--812.

\bibitem[\protect\citeauthoryear{Abernethy, Bartlett, and
  Hazan}{2011}]{abernethy11blackwell}
Abernethy, J.; Bartlett, P.~L.; and Hazan, E.
\newblock 2011.
\newblock Blackwell approachability and no-regret learning are equivalent.
\newblock In {\em {COLT} 2011 - The 24th Annual Conference on Learning Theory},
   27--46.

\bibitem[\protect\citeauthoryear{Berant and Liang}{2014}]{berant2014semantic}
Berant, J., and Liang, P.
\newblock 2014.
\newblock Semantic parsing via paraphrasing.
\newblock In {\em Proceedings of the 52nd Annual Meeting of the Association for
  Computational Linguistics},  1415--1425.

\bibitem[\protect\citeauthoryear{Berant \bgroup et al\mbox.\egroup
  }{2013}]{berant2013semantic}
Berant, J.; Chou, A.; Frostig, R.; and Liang, P.
\newblock 2013.
\newblock Semantic parsing on freebase from question-answer pairs.
\newblock In {\em {EMNLP} 2013},  1533--1544.

\bibitem[\protect\citeauthoryear{Brocker}{2009}]{brocker2009decomposition}
Brocker, J.
\newblock 2009.
\newblock Reliability, sufficiency, and the decomposition of proper scores.
\newblock {\em Quarterly Journal of the Royal Meteorological Society}
  135(643):1512--1519.

\bibitem[\protect\citeauthoryear{Buja, Stuetzle, and
  Shen}{2005}]{buja05lossfunctions}
Buja, A.; Stuetzle, W.; and Shen, Y.
\newblock 2005.
\newblock Loss functions for binary class probability estimation and
  classification: Structure and applications.

\bibitem[\protect\citeauthoryear{Cesa-Bianchi and
  Lugosi}{2006}]{cesabianchi2006prediction}
Cesa-Bianchi, N., and Lugosi, G.
\newblock 2006.
\newblock {\em Prediction, Learning, and Games}.
\newblock New York, NY, USA: Cambridge University Press.

\bibitem[\protect\citeauthoryear{Dawid}{1982}]{dawid1982well}
Dawid, A.~P.
\newblock 1982.
\newblock The well-calibrated bayesian.
\newblock {\em Journal of the American Statistical Association}
  77(379):605--610.

\bibitem[\protect\citeauthoryear{Devroye, Györfi, and
  Lugosi}{1996}]{devroye1996probabilistic}
Devroye, L.; Györfi, L.; and Lugosi, G.
\newblock 1996.
\newblock {\em {A} probabilistic theory of pattern recognition}.
\newblock Applications of mathematics. New York, Berlin, Heidelberg: Springer.

\bibitem[\protect\citeauthoryear{Foster and
  Vohra}{1998}]{foster98asymptoticcalibration}
Foster, D.~P., and Vohra, R.~V.
\newblock 1998.
\newblock Asymptotic calibration.

\bibitem[\protect\citeauthoryear{Foster}{1997}]{foster1997proof}
Foster, D.~P.
\newblock 1997.
\newblock {A Proof of Calibration Via Blackwell's Approachability Theorem}.
\newblock Discussion Papers 1182, Northwestern University.

\bibitem[\protect\citeauthoryear{Gneiting, Balabdaoui, and
  Raftery}{2007}]{gneiting2007probabilistic}
Gneiting, T.; Balabdaoui, F.; and Raftery, A.~E.
\newblock 2007.
\newblock Probabilistic forecasts, calibration and sharpness.
\newblock {\em Journal of the Royal Statistical Society: Series B}
  69(2):243--268.

\bibitem[\protect\citeauthoryear{Hazan and Kakade}{2012}]{hazan2012calibration}
Hazan, E., and Kakade, S.~M.
\newblock 2012.
\newblock (weak) calibration is computationally hard.
\newblock In {\em {COLT} 2012 - The 25th Annual Conference on Learning Theory,
  June 25-27, 2012, Edinburgh, Scotland},  3.1--3.10.

\bibitem[\protect\citeauthoryear{Jiang \bgroup et al\mbox.\egroup
  }{2012}]{jiang2012calibrating}
Jiang, X.; Osl, M.; Kim, J.; and Ohno{-}Machado, L.
\newblock 2012.
\newblock Calibrating predictive model estimates to support personalized
  medicine.
\newblock {\em {JAMIA}} 19(2):263--274.

\bibitem[\protect\citeauthoryear{Kuleshov and
  Liang}{2015}]{kuleshov2015calibrated}
Kuleshov, V., and Liang, P.
\newblock 2015.
\newblock Calibrated structured prediction.
\newblock In {\em Advances in Neural Information Processing Systems (NIPS)}.

\bibitem[\protect\citeauthoryear{Mannor and
  Stoltz}{2010}]{mannor2010calibration}
Mannor, S., and Stoltz, G.
\newblock 2010.
\newblock A geometric proof of calibration.
\newblock {\em Math. Oper. Res.} 35(4):721--727.

\bibitem[\protect\citeauthoryear{Menon \bgroup et al\mbox.\egroup
  }{2012}]{menon2012ranking}
Menon, A.~K.; Jiang, X.; Vembu, S.; Elkan, C.; and Ohno{-}Machado, L.
\newblock 2012.
\newblock Predicting accurate probabilities with a ranking loss.
\newblock In {\em 29th International Conference on Machine Learning,}.

\bibitem[\protect\citeauthoryear{Murphy}{1973}]{murphy1973vector}
Murphy, A.~H.
\newblock 1973.
\newblock A new vector partition of the probability score.
\newblock {\em Journal of Applied Meteorology} 12(4):595--600.

\bibitem[\protect\citeauthoryear{Nguyen and
  O'Connor}{2015}]{nguyen2015posterior}
Nguyen, K., and O'Connor, B.
\newblock 2015.
\newblock Posterior calibration and exploratory analysis for natural language
  processing models.
\newblock {\em CoRR} abs/1508.05154.

\bibitem[\protect\citeauthoryear{Niculescu-Mizil and
  Caruana}{2005}]{niculescu2005predicting}
Niculescu-Mizil, A., and Caruana, R.
\newblock 2005.
\newblock Predicting good probabilities with supervised learning.
\newblock In {\em Proceedings of the 22Nd International Conference on Machine
  Learning}, ICML '05.

\bibitem[\protect\citeauthoryear{Platt}{1999}]{platt1999probabilistic}
Platt, J.~C.
\newblock 1999.
\newblock Probabilistic outputs for support vector machines and comparisons to
  regularized likelihood methods.
\newblock In {\em ADVANCES IN LARGE MARGIN CLASSIFIERS},  61--74.
\newblock MIT Press.

\bibitem[\protect\citeauthoryear{Shalev-Shwartz}{2007}]{shalev2007phd}
Shalev-Shwartz, S.
\newblock 2007.
\newblock {\em Online Learning: Theory, Algorithms, and Applications}.
\newblock Phd thesis, Hebrew University.

\bibitem[\protect\citeauthoryear{{The Wellcome Trust Case Control
  Consortium}}{2007}]{wellcome2007genome}
{The Wellcome Trust Case Control Consortium}.
\newblock 2007.
\newblock {Genome-wide association study of 14,000 cases of seven common
  diseases and 3,000 shared controls}.
\newblock {\em Nature} 447(7145):661--678.

\bibitem[\protect\citeauthoryear{Vovk, Takemura, and
  Shafer}{2005}]{vovk2005defensive}
Vovk, V.; Takemura, A.; and Shafer, G.
\newblock 2005.
\newblock Defensive forecasting.
\newblock In {\em Proceedings of the Tenth International Workshop on Artificial
  Intelligence and Statistics, {AISTATS} 2005, Bridgetown, Barbados, January
  6-8, 2005}.

\bibitem[\protect\citeauthoryear{Yu, Li, and Deng}{2011}]{dong2011calibration}
Yu, D.; Li, J.; and Deng, L.
\newblock 2011.
\newblock Calibration of confidence measures in speech recognition.
\newblock {\em Trans. Audio, Speech and Lang. Proc.} 19(8):2461--2473.

\bibitem[\protect\citeauthoryear{Zadrozny and
  Elkan}{2002}]{zadrozny2002transforming}
Zadrozny, B., and Elkan, C.
\newblock 2002.
\newblock Transforming classifier scores into accurate multiclass probability
  estimates.
\newblock In {\em Eighth {ACM} Conference on Knowledge Discovery and Data
  Mining},  694--699.

\bibitem[\protect\citeauthoryear{Zhong and Kwok}{2013}]{zhong2013accurate}
Zhong, L.~W., and Kwok, J.~T.
\newblock 2013.
\newblock Accurate probability calibration for multiple classifiers.
\newblock IJCAI '13,  1939--1945.
\newblock AAAI Press.

\end{thebibliography}
}

%
%

\end{document}